\documentclass{article}
\usepackage{spconf,amsmath,graphicx}

\usepackage{amsthm}
\usepackage{amssymb}
\usepackage{graphicx} % more modern
\usepackage{color}

\usepackage{booktabs}

\usepackage{graphicx} % more modern
\usepackage[caption=false,font=footnotesize]{subfig}

\usepackage{algorithm}
\usepackage{algorithmic}

\DeclareMathOperator*{\argmin}{arg\,min}

\DeclareMathOperator{\trace}{tr}

\newcommand{\order}{\mathcal{O}} 
\newcommand{\x}{\mathbf{x}}
\renewcommand{\xi}{{\x}_{i}}
\newcommand{\e}{\mathbf{e}}
\newcommand{\X}{\mathbf{X}}
\newcommand{\Y}{\mathbf{Y}}
\newcommand{\K}{\mathbf{K}}
\newcommand{\y}{\mathbf{y}}
\newcommand{\U}{\mathbf{U}}

\newcommand{\R}{\mathbb{R}}    % for real numbers
\newcommand{\LLL}{\mathcal{L}}
\newcommand{\SSS}{\mathcal{S}}
\newcommand{\FFF}{\mathcal{F}}
\newcommand{\muu}{\boldsymbol{\mu}}
\newcommand{\Lam}{\boldsymbol{\Lambda}}
\newcommand{\Phii}{\boldsymbol{\Phi}}
\newcommand{\cc}{\mathbf{c}}
\newcommand{\CC}{\mathbf{C}}
\newcommand{\EE}{\mathbf{E}}

\theoremstyle{plain}
\newtheorem{thm}{\protect\theoremname}
\theoremstyle{plain}

\providecommand{\lemmaname}{Lemma}
\providecommand{\theoremname}{Theorem}

\title{A Randomized Approach to Efficient Kernel Clustering}
\name{Farhad Pourkamali-Anaraki$^{\star}$ \qquad Stephen Becker$^{\dagger}$}
\address{$^{\star}$ Department of Electrical, Computer, and Energy Engineering \\
	$^{\dagger}$ Department of Applied Mathematics\\ University of Colorado at Boulder, Boulder, CO 80309 USA\thanks{Copyright 2016 IEEE. Published in the IEEE 2016 Global Conference on Signal and Information Processing (GlobalSIP 2016), scheduled for Decemer 7-9, 2016 in Greater Washington, D.C., USA. Personal use of this material is permitted. However, permission to reprint/republish this material for advertising or promotional purposes or for creating new collective works for resale or redistribution to servers or lists, or to reuse any copyrighted component of this work in other works, must be obtained from the IEEE. Contact: Manager, Copyrights and Permissions / IEEE Service Center / 445 Hoes Lane / P.O. Box 1331 / Piscataway, NJ 08855-1331, USA. Telephone: + Intl. 908-562-3966.}}
\begin{document}
\ninept
\maketitle
\begin{abstract}
Kernel-based K-means clustering has gained popularity due to its simplicity and the power of its implicit non-linear representation of the data. A dominant concern is the memory requirement since memory scales as the square of the number of data points. We provide a new analysis of a class of approximate kernel methods that have more modest memory requirements, and propose a specific one-pass randomized kernel approximation followed by standard K-means on the transformed data. The analysis and experiments suggest the method is accurate, while requiring drastically less memory than standard kernel K-means and significantly less memory than Nystr\"om based approximations.
\end{abstract}
%---------------------------------------------------------------------------------------------------------------
\section{Introduction}
\label{sec:introduction}
Kernel-based approaches are popular methods for  supervised and unsupervised learning~\cite{LearningWithKernels}.
The $(i,j)$ entry of a kernel matrix $\K$ represents the inner product between the representations of data points $\x_i$ and $\x_j$ in a lifted space, and this lifting allows one to use linear techniques in the higher (or infinite) dimensional space which correspond to non-linear techniques in the original space. 
For example, the two classes of data in Fig.~\ref{fig:synthetic_original_data} are not linearly separable, but they become linearly separable after applying a suitable kernel (Fig.~\ref{fig:synthetic_mapped_data}).

With $n$ data points, the kernel matrix is $n\times n$, and for large modern data sets this is infeasible to store or compute with. For this reason, there is a long history of low-rank approximations of $\K$, starting with incomplete Cholesky factorizations in \cite{KatyaKernel2001}, and excellently summarized in \cite{BachKernelReview}. In particular, \cite{BachKernelReview} argues that the design of the approximation must go hand-in-hand with the learning algorithm, and proceeds to analyze the case of kernel regression.

The aim of this paper is to analyze kernel approximations in hard clustering and suggest a specific one-pass randomized kernel approximation. This approximation is based on a one-pass variant of the popular randomized approach described in \cite{Martinson_SVD}. Such an approach had previously been used to approximate a small inner $m \times m$ matrix in the Nystr\"om method~\cite{LiKwokNystrom2010}, but with accuracy limited by the large sampling requirements of the Nystr\"om approach. Instead, we precondition the kernel matrix in a streaming manner and \emph{then} sample, which allows us to take drastically fewer samples while maintaining excellent clustering performance (cf.~Table~\ref{table:synthetic_accuracy}).
Our algorithm is not necessarily faster than the Nystr\"om approach, but has lower memory requirements, such as around $10$ times lower memory for both the synthetic data in Table~\ref{table:synthetic_accuracy} and the real data in Fig.~\ref{fig:segment_data}.
A particular benefit to our proposal is that it consists of a distinct preprocessing phase followed by the standard K-means algorithm on transformed data, thus allowing one to leverage existing algorithm libraries.
\subsection{Notation}
\vspace*{-2mm}
We denote column vectors with lower-case bold letters and matrices with upper-case bold letters. Let $\|\K\|_F$ and $\|\K\|_2$ denote the Frobenius norm and spectral norm respectively. Moreover, $\|\K\|_*=\trace((\K^T\K)^{1/2})$ represents the trace norm, where $\trace(\cdot)$ is the trace operator. 

Also, we represent the entry in the $i$-th row and the $j$-th column of $\K$ as $[\K]_{ij}$. We let $\e_i$ denote the $i$-th vector of the canonical basis in $\R^K$, where entries are all zero except for the $i$-th one which is $1$.
\vspace*{-2mm}
\section{Preliminaries}\label{sec:prelim}
\vspace*{-2mm}
\subsection{K-means Clustering}\label{sec:K-means}
Consider a collection of $n$ data samples in $\R^p$, where $\X=[\x_1,\ldots,\x_n]\in\R^{p\times n}$ represents the data matrix. The K-means algorithm is a popular hard clustering technique that splits the data set into a known number of $K$ clusters. The resulting $K$-partition $\SSS=\{\SSS_1,\ldots,\SSS_K\}$ is a collection of $K$ non-empty pairwise disjoint sets that covers the data set. Moreover, each cluster $\SSS_k$ is represented using a vector $\muu_k\in\R^p$ that is associated with the $k$-th cluster.

Hence, the K-means objective is to find the optimal $K$-partition by minimizing the total sum of the squared Euclidean distances of each data sample to the closest cluster:
\begin{equation}
\FFF\left(\SSS\right)=\sum_{i=1}^{n}\sum_{k=1}^{K}t_{ik}\|\x_i-\muu_k\|_2^2\label{eq:K-means-objective}
\end{equation}
where $t_{ik}\in\{0,1\}$ is a binary indicator variable and $\mathbf{t}_i=[t_{i1},\ldots,t_{iK}]^T$ is the $k$-th canonical basis vector in $\R^K$ if and only if $\x_i$ belongs to the $k$-th cluster $\SSS_k$. 

Minimizing the objective function in~\eqref{eq:K-means-objective} is known to be NP-hard,
so the standard approach is to look for an approximate solution by an \emph{iterative} method that guarantees convergence only to a local minimum~\cite{Bishop}. In the first step, the assignment of data samples is updated with $\{\muu_k\}_{k=1}^{K}$ held fixed. In the next step, the cluster representatives $\{\muu_k\}_{k=1}^{K}$ are updated based on the most recent assignment.

K-means clustering works well only if all pairs of clusters are linearly separable, and does
  not perform well on finding non-linearly separable clusters of varying densities and distributions~\cite{Kernel_Kmeans_Dhillon}. A prominent approach to tackle this problem is to employ a non-linear distance function using the kernel trick~\cite{aizerman1964theoretical} in machine learning. The resulting algorithm is the so-called Kernel K-means which we explain in the next section. 
\vspace*{-3mm}
\subsection{The Kernel K-means Algorithm}\label{sec:Kernel-kmeans}
Kernel-based methods have provided a straightforward way to deal with non-linear structure in datasets. To be formal, each $\x_i$ is mapped to a higher dimensional feature space using the non-linear mapping $\Phii$, $\x_i\!\mapsto\!\!\Phii(\x_i)$ for $i\!\!=\!\!1,\ldots,\!n$. The kernel trick is based on the observation that many algorithms only need to compute the inner product between data points, and not see the data points themselves.  Thus the trick avoids the explicit mapping by allowing one to compute inner products between mapped data points in the feature space using a non-linear similarity measure used in Euclidean space $\R^p$: 
\begin{equation}
\langle\Phii(\x_i),\Phii(\x_j)\rangle=\kappa(\x_i,\x_j),\;\; \forall i, j\in\{1,\ldots,n\}
\end{equation}
where $\kappa(\cdot,\cdot)$ is a Mercer kernel function such that $\kappa$ induces a positive semidefinite matrix $[\K]_{ij}=\kappa(\x_i,\x_j)$ for all input data $\{\x_i\}_{i=1}^{n}$. Examples of such kernels include polynomial kernels $\kappa(\x,\y)=(\langle\x,\y\rangle+\gamma)^d$ and Gaussian radial basis function kernels $\kappa(\x,\y)=\exp(-\gamma\|\x-\y\|_2^2)$ with parameters $\gamma\in\R^+$ and $d\in\mathbb N$~\cite{Nonlinear_kernel_DL,anaraki2013kernel}.

The Kernel K-means algorithm finds a $K$-partition of the mapped data $\{\Phii(\x_i)\}_{i=1}^{n}$ by minimizing:  
\begin{equation}
\LLL\left(\SSS\right)=\sum_{i=1}^{n}\sum_{k=1}^{K}t_{ik}\|\Phii(\x_i)-\muu_k\|_2^2.\label{eq:Kernel-K-means-objective}
\end{equation}
The optimization problem of minimizing~\eqref{eq:Kernel-K-means-objective} can be solved using the same iterative procedure of K-means. To see this, consider the centroid of the $j$-th cluster $\muu_j=\frac{1}{|\SSS_j|}\sum_{\Phii(\x_l)\in\SSS_j}\Phii(\x_l)$.
This centroid cannot be computed explicitly, but 
we can compute the distance between each mapped data sample $\Phii(\x_i)$ and the centroid:
\begin{align}
\hspace{-1mm}& \hspace{-1mm} \|\Phii(\x_i)-\muu_j\|_2^2=\langle\Phii(\x_i)-\muu_j,\Phii(\x_i)-\muu_j\rangle \nonumber \\
\hspace{-1mm}&\hspace{-1mm} =[\K]_{ii}-\frac{2}{|\SSS_j|}\sum_{\Phii(\x_l)\in\SSS_j}\!\![\K]_{il}+\frac{1}{|\SSS_j|^2}\sum_{\Phii(\x_l),\Phii(\x_l')\in\SSS_j}\!\![\K]_{ll'}.
\end{align}
Hence, we see that Kernel K-means is an iterative algorithm that requires access to the full kernel matrix $\K$ without the need to explicitly map the data points. However, this clustering technique requires the storage and handling of a large kernel matrix $\K$ in each iteration. Therefore, the quadratic complexity of $\order(n^2)$ to store or $\order(n^2p)$ to compute the kernel matrix for each iteration makes the Kernel K-means algorithm non-scalable to large data sets.
\vspace*{-3mm}
\subsection{Prior Work on Efficient Kernel-Based Learning}\label{sec:Relation-Prior} 
Much research has focused on approximating the kernel matrix using a low-rank decomposition; more recent works focus on a sum of low-rank and diagonal or low-rank and sparse decompositions, but similar memory and computation considerations apply.
Note that the kernel matrix is a symmetric positive semidefinite matrix. Thus, its eigenvalue decomposition can be used to express a low-rank approximation:
\begin{equation}
\K\approx\U_r\Lam_r\U_r^T=\left(\U_r\Lam_r^{1/2}\right)\left(\Lam_r^{1/2}\U_r^T\right)=\Y^T\Y\label{eq:decom_kernel}
\end{equation}
where $\Lam_r\in\R^{r\times r}$, $r<n$, is a diagonal matrix containing the top $r$ eigenvalues of $\K$ in descending order and $\U_r\in\R^{n\times r}$ contains the associated orthonormal eigenvectors. Note that the eigenvalues of the kernel matrix are non-negative since $\K$ is positive semidefinite. The decomposition $\K\!\approx\!\Y^T\Y$ in~\eqref{eq:decom_kernel} essentially \emph{linearizes} the kernel matrix. Therefore, one can directly work with the new samples $\Y=[\y_1,\ldots,\y_n]$ embedded in Euclidean space $\R^r$, while respecting the non-linear similarities in the kernel matrix $\K$. Hence, this technique can be viewed as a preprocessing stage that eliminates the need to store and manipulate the kernel matrix during the learning process. Moreover, it is shown that the eigenvalue decomposition of the kernel matrix can be used to infer the number of clusters~\cite{girolami2002mercer}.

However, direct eigenvalue decomposition of large kernel matrices is often a demanding task that requires $\order(n^2)$ space and $\order(n^3)$ time. The popular Nystr{\"o}m method is an efficient algorithm to find low-rank approximations of large symmetric positive semidefinite matrices. The original Nystr{\"o}m method that was introduced in~\cite{Nystrom2001} is based on sampling a small subset of $m$ columns of $\K$ using uniform sampling without replacement. This one-pass algorithm was revisited in~\cite{Nystrom_Kernel_Approx}, where a data-dependent non-uniform sampling distribution was presented that requires at least two passes over the kernel matrix. The variants of the Nystr{\"o}m method have been proposed in the literature to analyze various sampling strategies, including~\cite{kumar2009sampling,zhang2010clusteredNys,si2014memory}. The recent paper~\cite{sun2015review} reviews different kinds of Nystr\"om methods for large-scale machine learning.
 
\vspace*{-3mm}
\section{Linearized Kernel K-means Clustering}\label{sec:Linear-kernel-kmeans}
In this section, we analyze the quality of Kernel K-means clustering under the low-rank decomposition of the kernel matrix $\K=\Y^T\Y$ given in~\eqref{eq:decom_kernel}. This low-rank decomposition can be done by using any low-rank approximation technique and we do not make any assumptions in this section. Our analysis compares the optimal solution of Kernel K-means on the new samples $\Y=[\y_1,\ldots,\y_n]\in\R^{r\times n}$ with the optimal solution of Kernel K-means on $\X=[\x_1,\ldots,\x_n]$. To do this, we first present an alternative formulation of the Kernel K-means objective function in equation~\eqref{eq:Kernel-K-means-objective}. This type of analysis that we follow in this paper is commonly used in the literature for K-means clustering, e.g.,~\cite{kmeans_plusplus,Randomized_Dim_K_means}. 

To begin, let us consider the matrix of mapped data samples $\Phii(\X)\!\!=\!\![\Phii(\x_1),\ldots,\Phii(\x_n)]$. We also define the cluster indicator matrix $\CC\!\!=\!\![\cc_1,\ldots,\cc_n]\!\!\in\!\R^{K\times n}$, where each row corresponds to a cluster. Each column of $\CC$ represents the cluster membership of $\Phii(\x_j)$ with only one nonzero entry such that $\cc_j\!\!=\!\!(1/\sqrt{|\SSS_i|})\e_i$ if and only if $\Phii(\x_j)$ belongs to the $i$-th cluster $\SSS_i$; denote the set of all such indicator matrices by $\mathcal{C}$. Given the matrix $\CC$, the matrix product $\Phii(\X)\CC^T$ consists of $K$ centroids as columns, where the $i$-th column is $\sqrt{|\SSS_i|}\muu_i$. Hence, $\Phii(\X)\CC^T\cc_j\!=\!(1/\sqrt{|\SSS_i|})(\sqrt{|\SSS_i|}\muu_i)\!=\!\muu_i$ selects the centroid of the $i$-th cluster $\SSS_i$ that $\Phii(\x_j)$ belongs to it. Thus, the Kernel K-means objective function~\eqref{eq:Kernel-K-means-objective} can be written as:
\begin{align}
\LLL(\CC)&  =\!\sum_{j=1}^{n}\|\Phii(\x_j)-\Phii(\X)\CC^T\cc_j\|_2^2 \!=\!\|\Phii(\X)-\Phii(\X)\CC^T\CC \|_F^2\nonumber \\
&=\!\trace\left((\mathbf{I}_{n\times n}-\CC^T\CC)\K(\mathbf{I}_{n\times n}-\CC^T\CC) \right)
\label{eq:KKmeans_Reform_Step1}
\end{align}
where we used $\|\mathbf{A}\|_F^2=\trace(\mathbf{A}^T\mathbf{A})$.

Next, we present some properties of the matrix $\CC$ defined above. Note that under the mild assumption that every cluster has at least one member, the cluster indicator matrix has $K$ orthonormal rows, i.e., $\CC\CC^T=\mathbf{I}_{K\times K}$. This follows from the normalization in our definition of $\CC$ and the fact that Kernel K-means is a hard clustering algorithm. As a result, we get $(\CC^T\CC)^2=\CC^T\CC$ which shows that both $\CC^T\CC$ and $(\mathbf{I}_{n\times n}-\CC^T\CC)$ are projection matrices. 

In the following theorem, we characterize the accuracy of Kernel K-means under the low-rank decomposition of the kernel matrix $\K$. In this case, one should replace the kernel matrix $\K$ with its low-rank decomposition $\Y^T\Y$ in the reformulated objective function~\eqref{eq:KKmeans_Reform_Step1}. This process can also be viewed as applying standard K-means on the new samples $\Y=[\y_1,\ldots,\y_n]$ in $\R^r$. Before stating the result, we emphasize that even though finding an optimal solution for K-means clustering is computationally difficult (NP-hard), the set of possible $K$-partitions of a finite number of data samples is finite. Thus, an optimal solution exists regardless of the difficulty in finding the optimal solution.
\vspace*{-2mm}
\begin{thm}
	Let $\CC^*$ be an optimal solution of Kernel K-means in feature space using the full kernel matrix $\K$:
	\begin{equation} 
	\CC^*\in \argmin_{\CC\in\mathcal{C}} \LLL(\CC)
	\end{equation}
	where $	\LLL(\CC)=\trace\left((\mathbf{I}_{n\times n}-\CC^T\CC)\K(\mathbf{I}_{n\times n}-\CC^T\CC)\right)$.
	Moreover, let $\widehat{\CC}$ be an optimal solution of the approximate Kernel K-means using 
	an approximation
	%the low-rank decomposition of the kernel matrix 
	$\widehat{\K}=\K-\EE$ with $\widehat{\K}=\Y^T\Y$ positive semidefinite:
	%$\K=\widehat{\K}+\EE=\Y^T\Y+\EE$:
	\begin{equation} 
	\widehat{\CC}\in\argmin_{\CC\in\mathcal{C}} \;\trace\left((\mathbf{I}_{n\times n}-\CC^T\CC)\Y^T\Y(\mathbf{I}_{n\times n}-\CC^T\CC)\right).
	\end{equation}
	Then, we have: 
	\begin{equation} \label{eq:result1}
	\LLL(\widehat{\CC})-\LLL(\CC^*)\leq 2\|\EE\|_*
	%\LLL(\widehat{\CC})-\LLL(\CC^*)\leq\trace(\EE)
	\end{equation}
	where $\EE$ is the low-rank approximation error of the kernel matrix $\K$.
	Furthermore, if $\widehat{\K}$ is the best rank $r$ approximation to $\K$, then
	$\EE$ is positive semidefinite and 
	we can improve Eq.~\eqref{eq:result1} to
	\begin{equation} \label{eq:result2}
	\LLL(\widehat{\CC})-\LLL(\CC^*)\leq\trace(\EE).
	\end{equation}
\end{thm}
% -----
\begin{proof}
The proof follows from the properties of the cluster indicator matrix $\CC$ and applying H\"older's inequality, i.e., the trace norm and spectral norm are dual. 
\end{proof}
This theorem indicates that the optimal objective value under the approximate matrix is not far from the true objective value. The bound is tight to within a small constant, as one can construct adversarial examples in any dimension where 
$\LLL(\widehat{\CC})-\LLL(\CC^*)\ge \frac{1}{2}\|\EE\|_*$.
    \vspace{-.3cm}
\section{The Proposed Method}\label{sec:Efficient-kernel-kmeans} 
    \vspace{-.3cm}
In this section, we present an efficient method for Kernel K-means clustering on large-scale data sets. The key component of our method is to use highly efficient and accurate low-rank approximation techniques that require just a \emph{single pass} over the kernel matrix to eliminate the need to store or recompute large kernel matrices. In~\cite{Martinson_SVD}, a class of randomized algorithms were proposed to construct low-rank approximations of large matrices, and tail bounds given on the spectral and Frobenius norm of the error $\EE=\K-\widehat{\K}$ (from which one can bound $\|\EE\|_*$ as well). We show how these randomized methods can be employed to achieve improved performance compared to the standard Nystr{\"o}m approximation in applications with limited memory resources. 

First, we briefly explain the one-pass eigenvalue decomposition algorithm from~\cite{Martinson_SVD}. Given a symmetric matrix $\K\in\R^{n\times n}$, the first step is to find a good basis for both the column space and the row space of $\K$. To this end, a standard Gaussian random matrix $\boldsymbol{\Omega}\!\in\!\R^{n\times r'}$ whose entries are i.i.d.~$\mathcal{N}(0,1)$ variables is generated with $r'\!\!=\!\!(r+l)$ for some rank $r$ and oversampling $l$. The oversampling parameter is often used to increase the accuracy of the method. Then, $\mathbf{W}\!\!=\!\!
\K\boldsymbol{\Omega}$ is computed and one finds $\mathbf{Q}\!\in\!\R^{n\times r}$ whose columns form an orthonormal basis for the range of $\mathbf{W}$; this can be done by computing the $r$ leading left singular vectors of $\mathbf{W}$ or via the QR decomposition. Therefore, we have $\K\!\approx\!\mathbf{Q}(\mathbf{Q}^T\K\mathbf{Q})\mathbf{Q}^T$ for which a naive approach requires one more pass over $\K$ to compute $\mathbf{B}=\mathbf{Q}^T\K\mathbf{Q}$ and find its low-rank decomposition. However, the matrix $\mathbf{B}$ can be computed by solving the equation $\mathbf{B}(\mathbf{Q}^T\boldsymbol{\Omega})\!=\!(\mathbf{Q}^T\mathbf{W})$ without revisiting $\K$~\cite{Martinson_SVD}. Finally, the eigenvalue decomposition of $\mathbf{B}\in\R^{r\times r}$ yields the desired rank $r$ approximation of $\K$.   

The major drawback of this approach is the memory and computation burden associated with the Gaussian random matrix $\boldsymbol{\Omega}$. Computing the matrix $\mathbf{W}$ takes $\order(n^2r')$ time which is quadratic in the number of samples. To address this problem, the Gaussian random matrix is replaced with a much more efficient structured random matrix $\boldsymbol{\Omega}=\mathbf{D}\mathbf{H}\mathbf{R}$~\cite{ImprovedAnalysis,Martinson_SVD}. The matrix $\mathbf{D}=\mathbf{D}^T\in\R^{n\times n}$ is a stochastic diagonal matrix whose entries on the main diagonal are random variables drawn uniformly from $\{\pm1\}$. 
The matrix $\mathbf{H}=\mathbf{H}^T\in\R^{n\times n}$ is the Hadamard matrix for which matrix-vector multiplication can be implemented in $\order(n\log(n))$ complexity, hence it is 
inexpensive to multiply and store compared to the Gaussian matrix. The matrix $\mathbf{H}$ is not stored explicitly, and applying $\mathbf{H}$ to a matrix is efficient in parallel and distributed environments; our implementation uses the pthread library and sees a $11$ times speedup over the non-parallel version when using $16$ threads.
The sub-sampling matrix $\mathbf{R}\in\R^{n\times r'}$ consists of $r'$ columns of the identity matrix $\mathbf{I}_{n\times n}$ drawn uniformly random without replacement. 
\begin{algorithm}[tb]
	\caption{One-Pass Kernel K-means}
	\label{alg:OnePassEig}
	\textbf{Input:} kernel matrix $\K\in\R^{n\times n}$, rank $r$, oversampling $l$, number of clusters $K$, number of iterations 
	
	\textbf{Output:} cluster indicator matrix $\CC$
	\begin{algorithmic}[1]
		\STATE $r'\leftarrow r+l$,\hspace{2mm} $\mathbf{R}\in\R^{n\times r'}$: random sampling matrix
		\STATE $\mathbf{W}\in\R^{n\times r'}\leftarrow(\mathbf{R}^T\mathbf{H}\mathbf{D}\K)^T$
		\STATE find an orthonormal matrix $\mathbf{Q}\in\R^{n\times r}$ by QR decomposition or $r$ leading left singular vectors of $\mathbf{W}$
		\STATE solve $\mathbf{B}(\mathbf{Q}^T\boldsymbol{\Omega})=(\mathbf{Q}^T\mathbf{W})$
		\STATE $\mathbf{B}=\mathbf{V}\boldsymbol{\Sigma}\mathbf{V}^T$
		\STATE $\Y=\boldsymbol{\Sigma}^{1/2}\mathbf{V}^T\mathbf{Q}^T\in\R^{r\times n}$
		\STATE perform standard K-means on $\Y=[\y_1,\ldots,\y_n]$ in $\R^r$ 
	\end{algorithmic}
\end{algorithm}

Our proposed efficient Kernel K-means clustering method is presented in Alg.~\ref{alg:OnePassEig}. We emphasize that our method requires only one pass over the columns of the kernel matrix $\K$, and that batches of columns of $\K$ can be constructed on-the-fly, so the entire kernel matrix $\K$ is never formed in memory: the minimal memory requirement is $\order(r'n)$.
Furthermore, the computation cost in each iteration of K-means is $\order(rnK)$ and the rank $r$ determines the error in the low-rank approximation of the kernel matrix $\K$. 

We also compare various aspects of our proposed approach with the Nystr{\"o}m method. The basic idea behind the Nystr{\"o}m method is to sample $m$ columns from the kernel matrix $\K$.
It is obvious that the more columns are sampled, the more accurate the resulting low-rank approximation is. The sampling strategy plays an important role in the accuracy of the Nystr{\"o}m method. The basic method originally proposed by~\cite{Nystrom2001} is a one-pass algorithm that employs uniform sampling without replacement.
Hence, our method in this paper is similar to the standard one-pass Nystr{\"o}m method as they both use uniform sampling without replacement. However, our method takes advantage of the preconditioning transformation of the kernel matrix $\K\mapsto(\mathbf{H}\mathbf{D})\K$ before sub-sampling. As shown in~\cite{ImprovedAnalysis}, the transformation $\mathbf{H}\mathbf{D}$ equilibrates row norms which eliminates the necessity to use more sophisticated sampling strategies. 
\begin{figure}[t]
	\vskip 0.2in
	\begin{center}
		\centerline{\includegraphics[width=0.55\columnwidth]{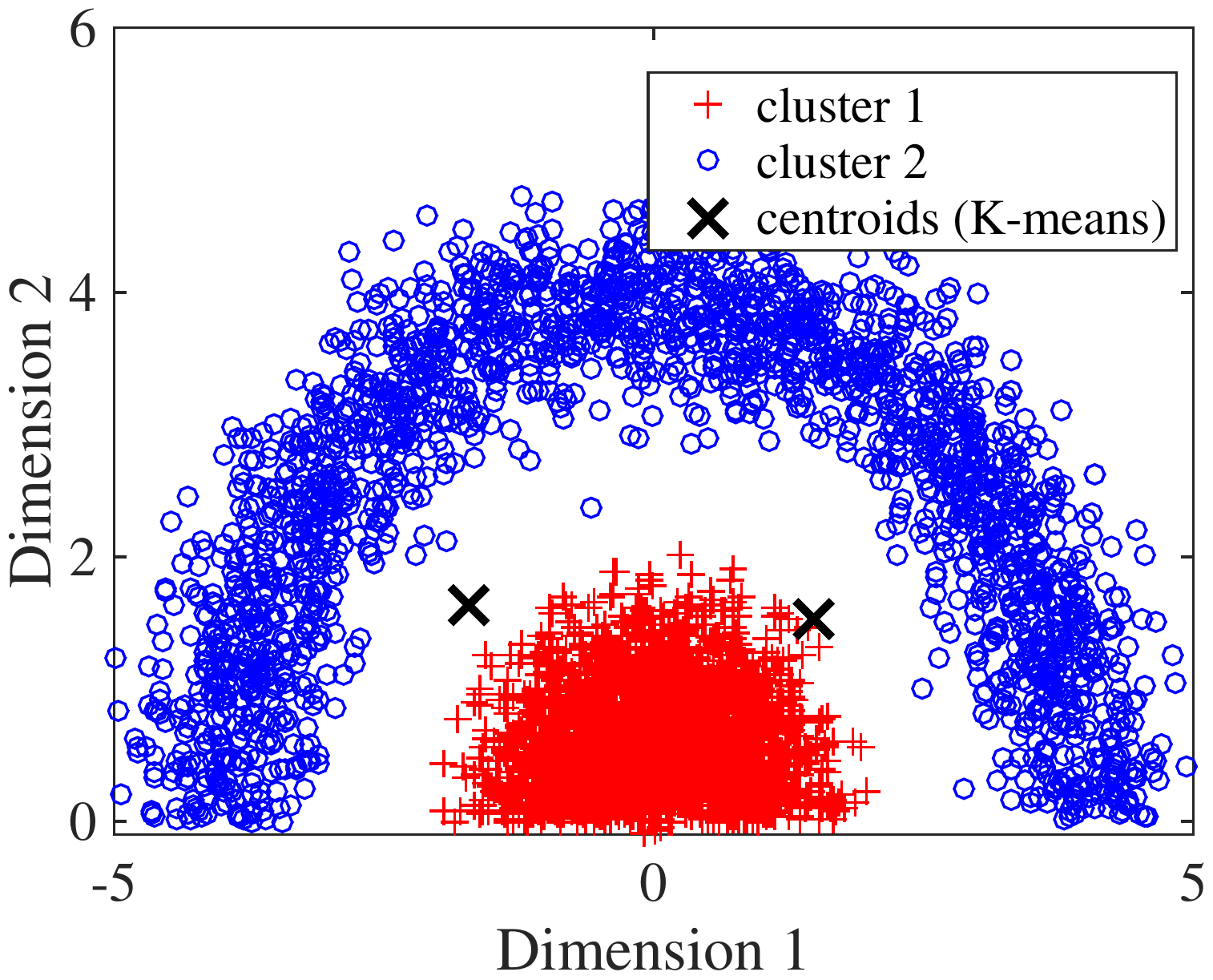}}
		\vspace*{-4mm}
		\caption{Original data. The centroids estimated by K-means are shown, which are clearly unhelpful at finding the true clusters.}
		\label{fig:synthetic_original_data}
	\end{center}
	\vskip -0.4in
\end{figure}

To gain some intuition, we consider a synthetic data set shown in Fig.~\ref{fig:synthetic_original_data}. This data set consists of $n\!=\!4000$ samples in $\R^2$ that are non-linearly separable but not linearly separable. Thus,  standard K-means is not able to identify these two clusters, and the two centroids selected by  standard K-means do not describe the true clusters.
\begin{figure}[t]
	\begin{centering}
		\subfloat[Exact Decomp.]{\begin{centering}
				\includegraphics[width=0.48\columnwidth]{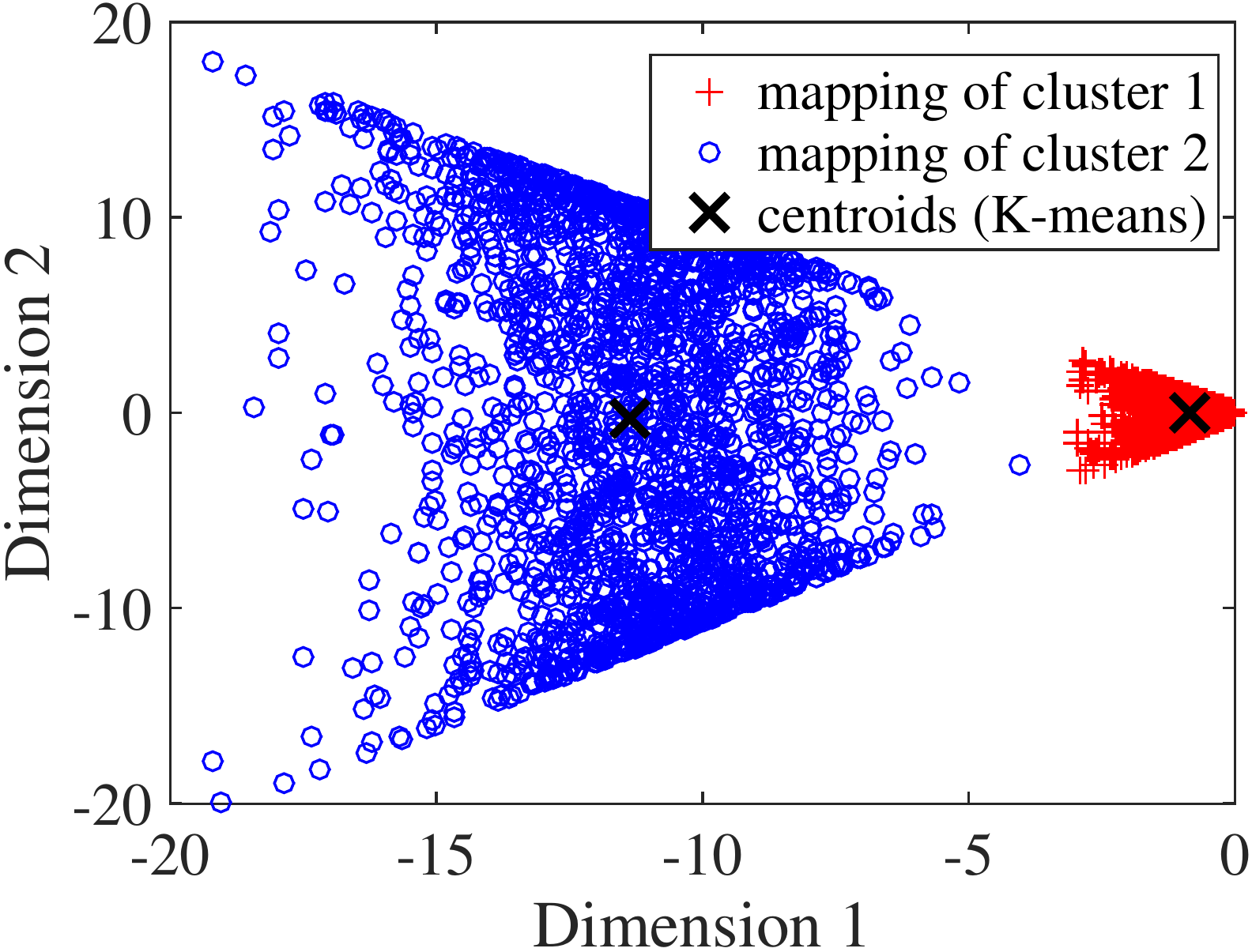}
				\par\end{centering}
		}\subfloat[Our Method]{\begin{centering}
			\includegraphics[width=0.48\columnwidth]{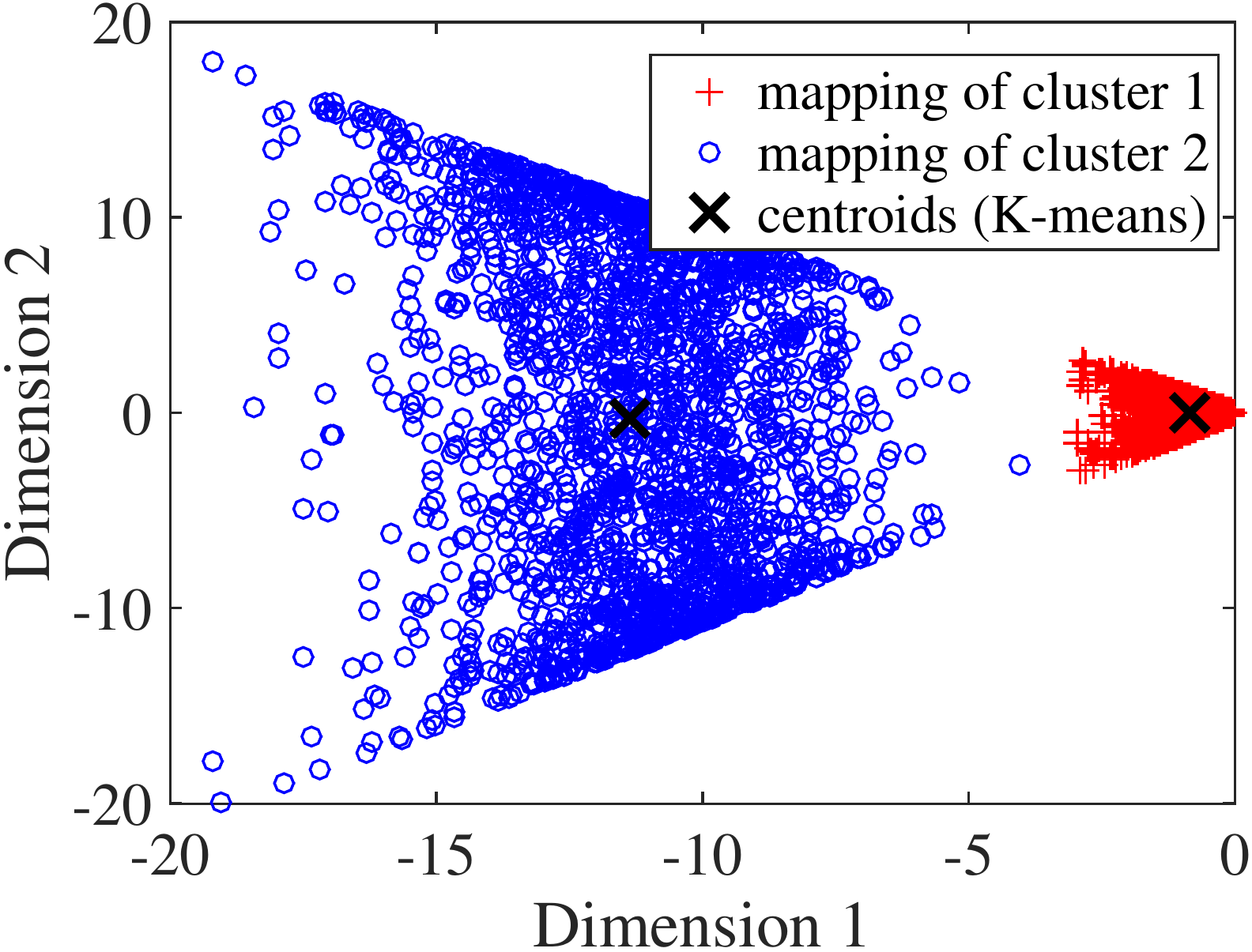}
			\par\end{centering}
	}
	\vspace*{-4mm}
	\subfloat[Nystr{\"o}m, $m$=20 ]{\begin{centering}
		\includegraphics[width=0.48\columnwidth]{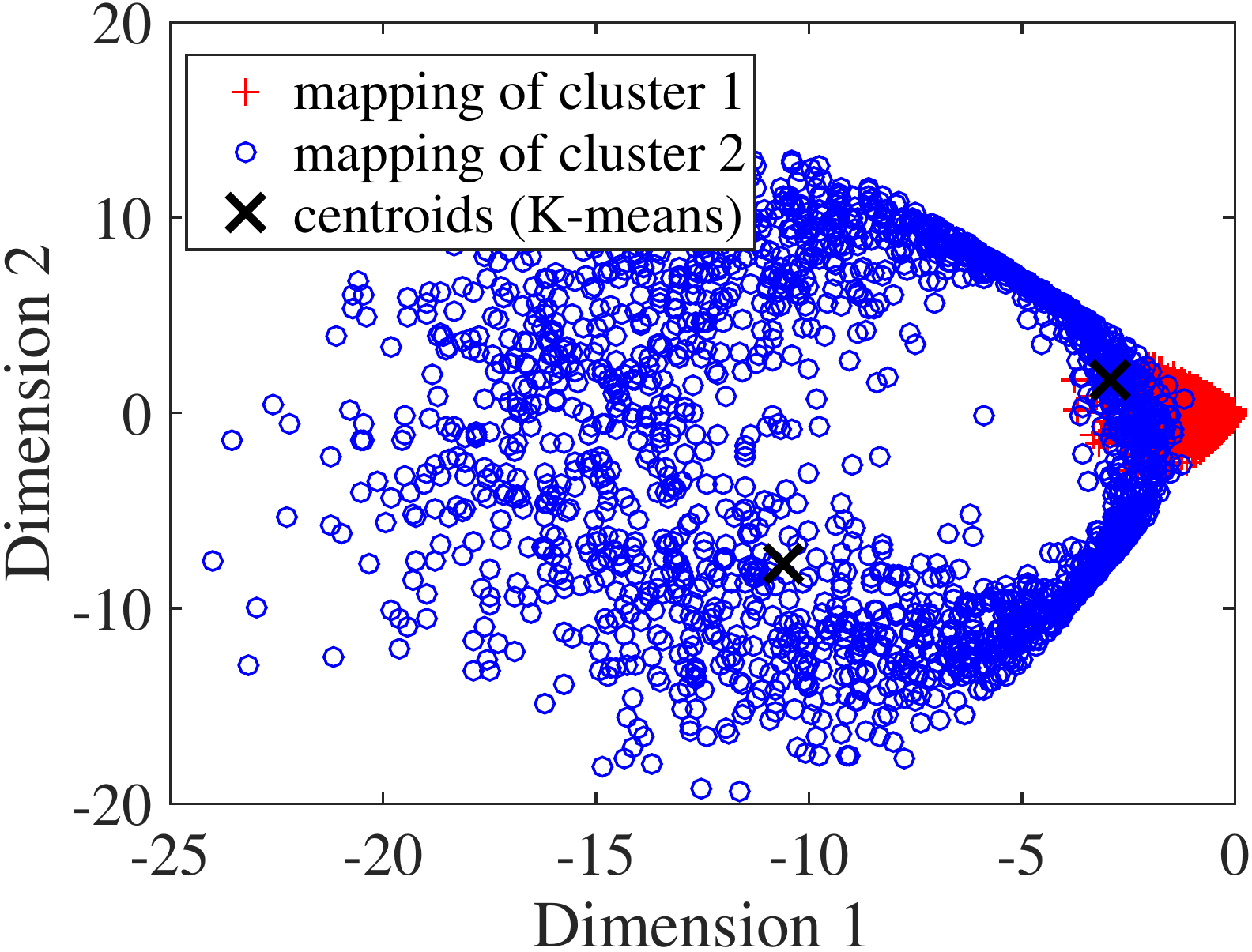}
		\par\end{centering}
}\subfloat[Nystr{\"o}m, $m$=100]{\begin{centering}
	\includegraphics[width=0.48\columnwidth]{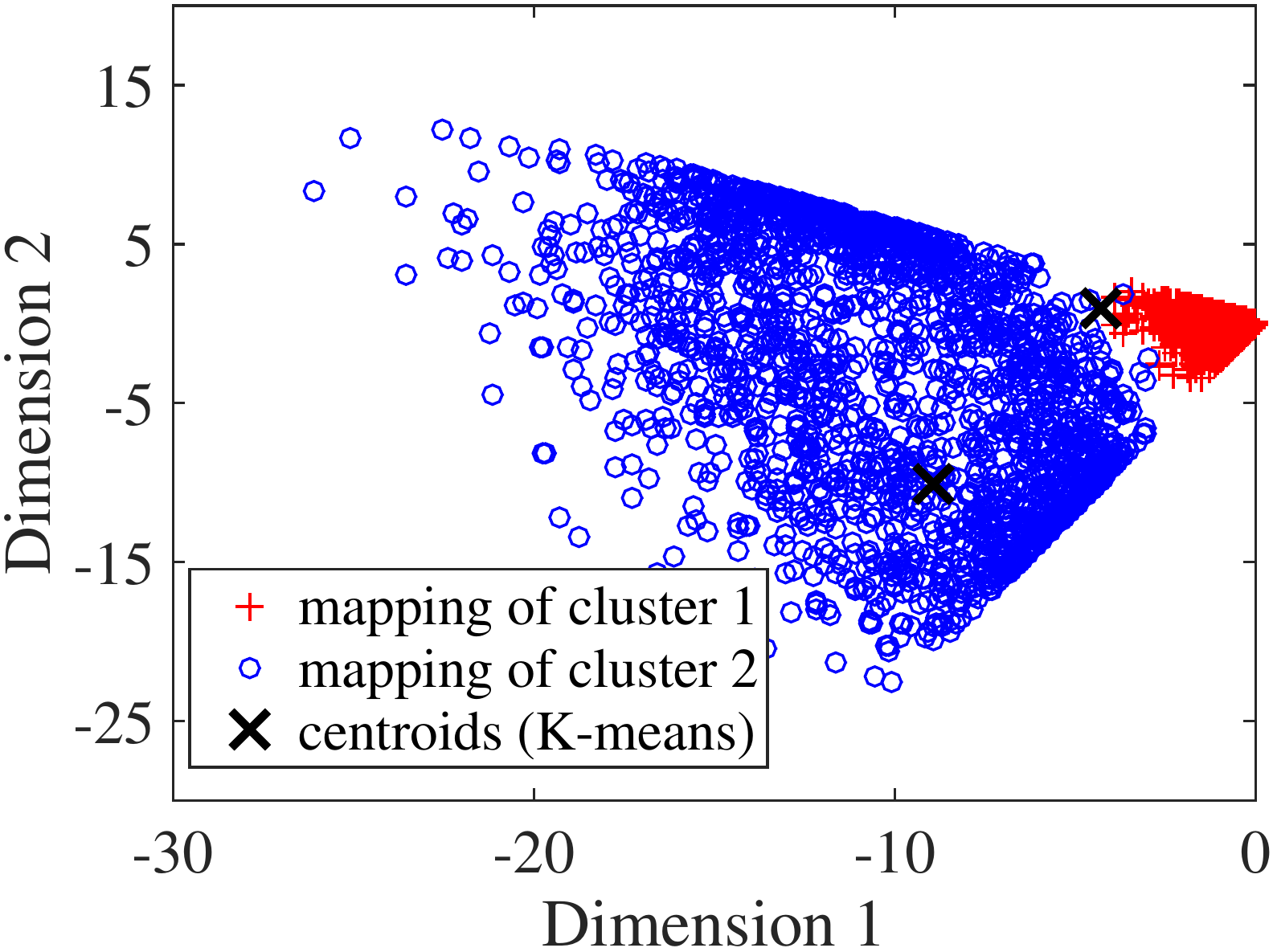}
	\par\end{centering}
}	
\par\end{centering}
\vspace*{-2mm}
\caption{Mapping of the original data using the low-rank approximation $\K\approx\Y^T\Y$. Applying K-means on the mapped data $\Y$ obtained by (a) the exact eigenvalue decomposition, or (b) our method, identifies the two underlying clusters accurately. 
	\label{fig:synthetic_mapped_data}}
\end{figure} 

We compute the kernel matrix $\K$ whose elements are obtained by using the polynomial kernel of order $d\!=\!2$, i.e., $\kappa(\x_i,\x_j)\!=\!\langle\x_i,\x_j\rangle^2$. The rank parameter is $r\!=\!2$ and our goal is to find $\Y\!\in\!\R^{2\times n}$ using the low-rank approximation of the kernel matrix. Our proposed method with the oversampling parameter $l\!=\!10$ is compared with the standard Nystr{\"o}m method with 
both $m\!=\!20$ and $m\!=\!100$ sampled columns.
We show the visualization of the new samples $\{\y_i\}_{i=1}^{4000}$ in Fig.~\ref{fig:synthetic_mapped_data}. Note that the accuracy of our approach is almost identical to the exact $r\!=\!2$ eigenvalue decomposition of the kernel matrix. Moreover, we see that our method is able to identify these two clusters using the standard K-means on the mapped samples $\Y$. However, the Nystr{\"o}m method does not provide an accurate solution even for a large value of $m=100\approx8r'$. We also compare the normalized kernel approximation error defined as $\|\K-\widehat{\K}\|_F/\|\K\|_F$ and the clustering accuracy in Table~\ref{table:synthetic_accuracy}. We see that our method outperforms the Nystr{\"o}m technique on this data set. 
\begin{table}[t]
	\caption{Accuracy of Kernel K-means methods, on data from Fig.~\ref{fig:synthetic_original_data}, $r\!=\!2$. For reference, (non-kernel) K-means has only $0.53$ accuracy. Our method takes the equivalent of $m\!=\!12$ columns in the Nystr{\"o}m approach.
	}
	\label{table:synthetic_accuracy}
%	\vskip 0.15in
    \vspace{-.4cm}
	\scriptsize
	\begin{center}
				\begin{tabular}{lccc}%{|l|c|c|cr}
					%					\hline
					%					\abovespace\belowspace
%					\hline
                    \toprule 
					Method & \vtop{\hbox{\strut Kernel Approx.}\hbox{\strut Error}} & \vtop{\hbox{\strut Clustering}\hbox{\strut Accuracy}}\\
					%					\hline
					%					\abovespace
%					\hline
                    \midrule
					Exact Decomposition  & 0.40& 0.99\\
					Our Method & 0.40& 0.99\\
					Nystr{\"o}m, $m$=20     & 0.56& 0.74\\
					Nystr{\"o}m, $m$=100      & 0.44& 0.75      \\
					%					\hline
%					\hline
                    \bottomrule
				\end{tabular}
	\end{center}
%	\vskip -0.1in
    \vspace{-.6cm}
\end{table}

Finally, we present the experimental evaluation of our method in Alg.~\ref{alg:OnePassEig} on a real-world data set. Our proposed approach is implemented in MATLAB (with the Hadamard code in C/mex) and compared against the standard Nystr{\"o}m method as well as the exact eigenvalue decomposition. We use the MATLAB \texttt{kmeans} function with $10$ different initializations, the maximum number of iterations is set to $20$ and $r=2$ is used for low-rank approximations of kernel matrices. Since Nystr{\"o}m and our method are stochastic, we re-run each experiment $100$ times and report the average over these trials.

We consider the image segmentation data set that can be downloaded from the UCI Repository. This data set contains $n=2310$ instances from $K=7$ outdoor images. Each instance represents a $3\times3$ region with $p=19$ attributes that are normalized to unit $\ell_2$ norm. Here, we choose the homogeneous polynomial kernel of order $d=2$, i.e., $\kappa(\x_i,\x_j)=\langle\x_i,\x_j\rangle^2$.

Fig.~\ref{fig:segment_data}(a) shows the normalized approximation error of the kernel matrix $\|\K-\widehat{\K}\|_F/\|\K\|_F$ for varying number of sampled columns $m$ and fixed oversampling parameter $l=5$ in our method. As we see, sampling $r'=2+5=7$ rows of the preconditioned kernel matrix $(\mathbf{H}\mathbf{D})\K$ leads to a more accurate decomposition that sampling $m=50\approx7r'$ columns of the kernel matrix $\K$ in the Nystr{\"o}m  method. Moreover, the accuracy of our approach is very close to the optimal exact eigenvalue decomposition. 

In Fig.~\ref{fig:segment_data}(b), the clustering accuracy of our method is compared with the other kernel clustering techniques. Again, we see that our approach has higher accuracy than the the Nystr{\"o}m decomposition approach. In this example, the accuracy of full Kernel K-means ($r=\text{rank}(\K)$) is $0.46$ and both our method and the approximate Kernel K-means using the exact eigenvalue decomposition (with $r=2$) have higher accuracy than the full Kernel K-means. 
\begin{figure}[t]
    \vspace{-.4cm}
	\begin{centering}
		\subfloat[Kernel Approx. Error]{\begin{centering}
				\includegraphics[width=.50\columnwidth]{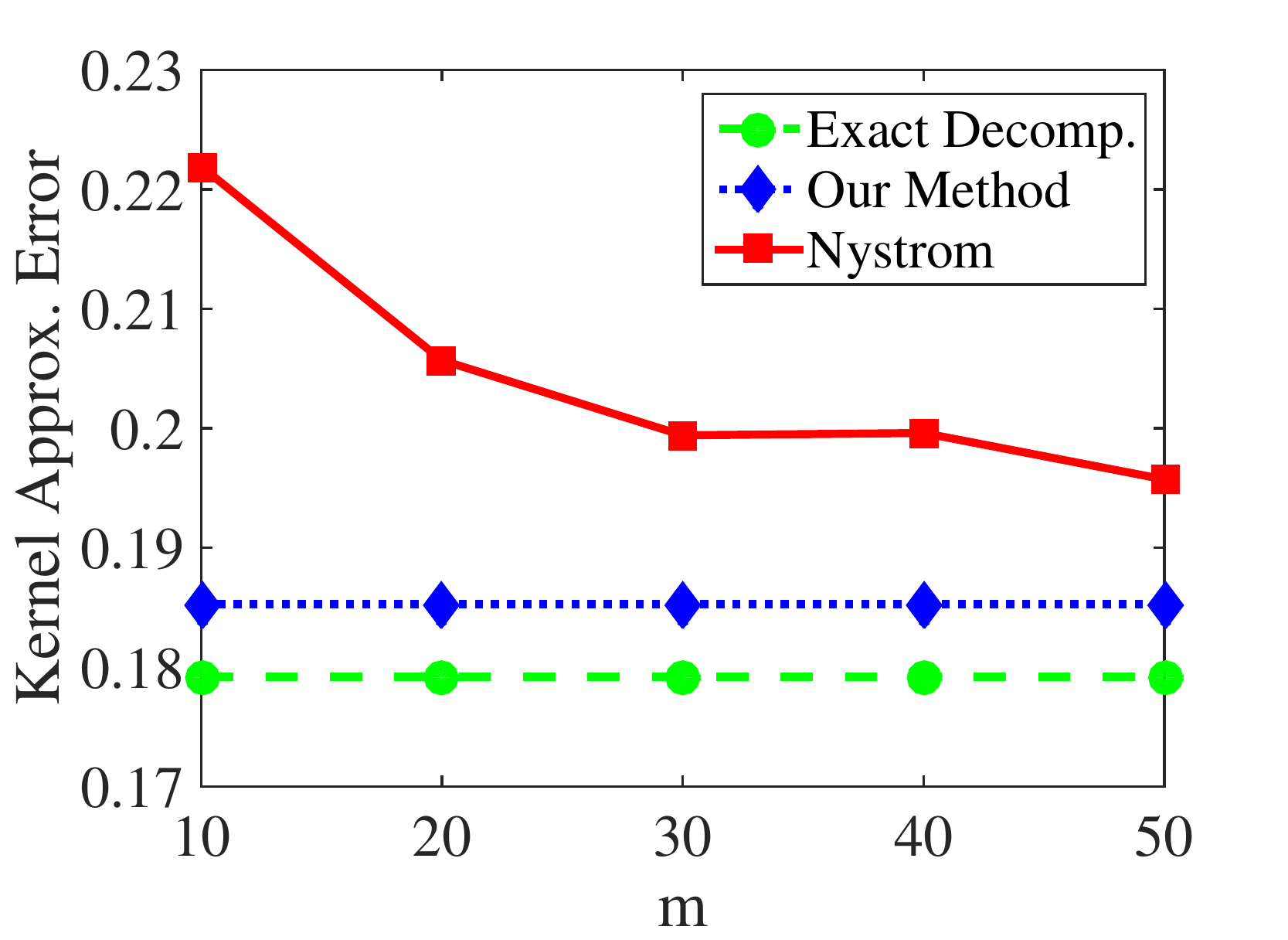}
				\par\end{centering}
		}
		\subfloat[Clustering Accuracy]{\begin{centering}
				\includegraphics[width=.50\columnwidth]{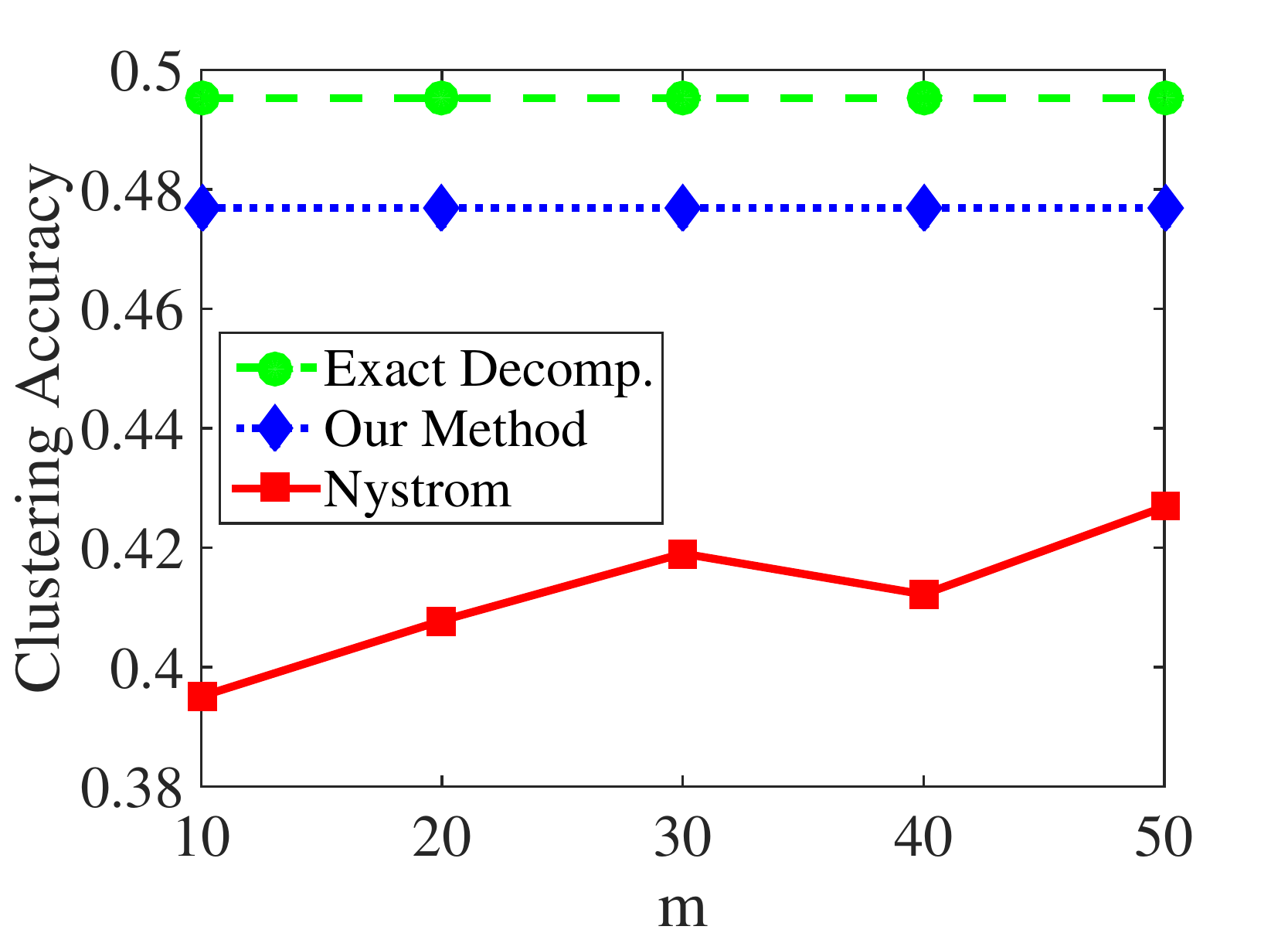}
				\par\end{centering}
		}	
		\par\end{centering}
	\caption{Results for image segmentation data set. Plot of (a) normalized approximation error of the kernel matrix (b) kernel clustering accuracy for varying number of sampled columns $m$.
		\label{fig:segment_data}}
\end{figure} 
    \vspace{-.3cm}
\section{Conclusions}\label{sec:conclusions}
\vspace{-.3cm}
We considered a class of approximate Kernel K-means algorithms in which the kernel matrix is replaced by its low-rank approximation. Our theoretical analysis provides insights into the effect of the approximation on the objective function of Kernel K-means, showing that the optimal objective value under the low-rank approximation is not far from the true objective value. 
Our theoretical result is applicable to any low-rank approximation technique. 

Furthermore, we introduced a specific one-pass randomized algorithm for Kernel K-means. 
Some benefits of our approach are ease of implementation, tunable accuracy vs.~memory/speed tradeoff using the parameter $r$, and low-memory requirements. The parameter $r$ is typically chosen with cross-validation on a subset of data. 
 %-------------------------------------------------------------------------------------------------------------
\bibliographystyle{IEEEbib}
\bibliography{phd_farhad}

\end{document}